\documentclass[a4paper,12pt]{amsart}

\usepackage{graphicx,natbib}

\newtheorem{thm}{Theorem}
\newtheorem{lem}{Lemma}
\newtheorem{cor}{Corollary}

\def\ci{\!\perp\!}
\def\nci{\!\not\perp\!}

\begin{document}

\title[]{Finding Consensus Bayesian Network Structures}

\author[]{Jose M. Pe\~{n}a\\
ADIT, Department of Computer and Information Science\\
Link\"oping University, SE-58183 Link\"{o}ping, Sweden\\
jose.m.pena@liu.se}

\date{\today}

\begin{abstract}
Suppose that multiple experts (or learning algorithms) provide us with alternative Bayesian network (BN) structures over a domain, and that we are interested in combining them into a single consensus BN structure. Specifically, we are interested in that the consensus BN structure only represents independences all the given BN structures agree upon and that it has as few parameters associated as possible. In this paper, we prove that there may exist several non-equivalent consensus BN structures and that finding one of them is NP-hard. Thus, we decide to resort to heuristics to find an approximated consensus BN structure. In this paper, we consider the heuristic proposed in \citep{MatzkevichandAbramson1992,MatzkevichandAbramson1993a,MatzkevichandAbramson1993b}. This heuristic builds upon two algorithms, called Methods A and B, for efficiently deriving the minimal directed independence map of a BN structure relative to a given node ordering. Methods A and B are claimed to be correct although no proof is provided (a proof is just sketched). In this paper, we show that Methods A and B are not correct and propose a correction of them.
\end{abstract}

\maketitle

\section{Introduction}\label{sec:intro}

Bayesian networks (BNs) are a popular graphical formalism for representing probability distributions. A BN consists of structure and parameters. The structure, a directed and acyclic graph (DAG), induces a set of independencies that the represented probability distribution satisfies. The parameters specify the conditional probability distribution of each node given its parents in the structure. The BN represents the probability distribution that results from the product of these conditional probability distributions. Typically, a single expert (or learning algorithm) is consulted to construct a BN of the domain at hand. Therefore, there is a risk that the so-constructed BN is not as accurate as it could be if, for instance, the expert has a bias or overlooks certain details. One way to minimize this risk consists in obtaining multiple BNs of the domain from multiple experts and, then, combining them into a single consensus BN. This approach has received significant attention in the literature \citep{MatzkevichandAbramson1992,MatzkevichandAbramson1993a,MatzkevichandAbramson1993b,Maynard-ReidIIandChajewska2001,NielsenandParsons2007,PennockandWellman1999,RichardsonandDomingos2003,delSagradoandMoral2003}. The most relevant of these references is probably \citep{PennockandWellman1999}, because it shows that even if the experts agree on the BN structure, no method for combining the experts' BNs produces a consensus BN that respects some reasonable assumptions and whose structure is the agreed BN structure. Unfortunately, this problem is often overlooked. To avoid it, we propose to combine the experts' BNs in two steps. First, finding the consensus BN structure and, then, finding the consensus parameters for the consensus BN structure. This paper focuses only on the first step. Specifically, we assume that multiple experts provide us with alternative DAG models of a domain, and we are interested in combining them into a single consensus DAG. Specifically, we are interested in that the consensus DAG only represents independences all the given DAGs agree upon and as many of them as possible. In other words, the consensus DAG is the DAG that represents the most independences among all the minimal directed independence (MDI) maps of the intersection of the independence models induced by the given DAGs.\footnote{It is worth mentioning that the term consensus DAG has a different meaning in computational biology \citep{Jacksonetal.2005}. There, the consensus DAG of a given set of DAGs $G^1, \ldots, G^m$ is defined as the DAG that contains the most of the arcs in $G^1, \ldots, G^m$. Therefore, the difficulty lies in keeping as many arcs as possible without creating cycles. Note that, unlike in the present work, a DAG is not interpreted as inducing an independence model in \citep{Jacksonetal.2005}.} To our knowledge, whether the consensus DAG can or cannot be found efficiently is still an open problem. See \citep{MatzkevichandAbramson1992,MatzkevichandAbramson1993a,MatzkevichandAbramson1993b} for more information. In this paper, we redefine the consensus DAG as the DAG that has the fewest parameters associated among all the MDI maps of the intersection of the independence models induced by the given DAGs. This definition is in line with that of finding a DAG to represent a probability distribution $p$. The desired DAG is typically defined as the MDI map of $p$ that has the fewest parameters associated rather than as the MDI map of $p$ that represents the most independences. See, for instance, \citep{Chickeringetal.2004}. The number of parameters associated with a DAG is a measure of the complexity of the DAG, since it is the number of parameters required to specify all the probability distributions that can be represented by the DAG.

In this paper, we prove that there may exist several non-equivalent consensus DAGs and that finding one of them is NP-hard. Thus, we decide to resort to heuristics to find an approximated consensus DAG. In this paper, we consider the following heuristic due to \citet{MatzkevichandAbramson1992,MatzkevichandAbramson1993a,MatzkevichandAbramson1993b}. First, let $\alpha$ denote any ordering of the nodes in the given DAGs, which we denote here as $G^1, \ldots, G^m$. Then, find the MDI map $G^i_{\alpha}$ of each $G^i$ relative to $\alpha$. Finally, let the approximated consensus DAG be the DAG whose arcs are exactly the union of the arcs in $G^1_{\alpha}, \ldots, G^m_{\alpha}$. It should be mentioned that our formulation of the heuristic differs from that in \citep{MatzkevichandAbramson1992,MatzkevichandAbramson1993a,MatzkevichandAbramson1993b} in the following two points. First, the heuristic was introduced under the original definition of consensus DAG. We justify later that the heuristic also makes sense under our definition of consensus DAG. Second, $\alpha$ was originally required to be consistent with one of the given DAGs. We remove this requirement. All in all, a key step in the heuristic is finding the MDI map $G^i_{\alpha}$ of each $G^i$. Since this task is not trivial, \citet{MatzkevichandAbramson1993b} present two algorithms, called Methods A and B, for efficiently deriving $G^i_{\alpha}$ from $G^i$. Methods A and B are claimed to be correct although no proof is provided (a proof is just sketched). In this paper, we show that Methods A and B are not correct and propose a correction of them.

As said, we are not the first to study the problem of finding the consensus DAG. In addition to the works discussed above by \citet{MatzkevichandAbramson1992,MatzkevichandAbramson1993a,MatzkevichandAbramson1993b} and \citet{PennockandWellman1999}, some other works devoted to this problem are \citep{Maynard-ReidIIandChajewska2001,NielsenandParsons2007,RichardsonandDomingos2003,delSagradoandMoral2003}. We elaborate below on the differences between these works and ours. \citet{Maynard-ReidIIandChajewska2001} propose to adapt existing score-based algorithms for learning DAGs from data to the case where the learning data is replaced by the BNs provided by some experts. Their approach suffers the problem pointed out by \citet{PennockandWellman1999}, because it consists essentially in learning a consensus DAG from a combination of the given BNs. A somehow related approach is proposed by \citet{RichardsonandDomingos2003}. Specifically, they propose a Bayesian approach to learning DAGs from data, where the prior probability distribution over DAGs is constructed from the DAGs provided by some experts. Since their approach requires data and does not combine the given DAGs into a single DAG, it addresses a problem rather different from the one in this paper. Moreover, the construction of the prior probability distribution over DAGs ignores the fact that some given DAGs may be different but equivalent. That is, unlike in the present work, a DAG is not interpreted as inducing an independence model. A work that is relatively close to ours is that by \citet{delSagradoandMoral2003}. Specifically, they show how to construct a MDI map of the intersection and union of the independence models induced by the DAGs provided by some experts. However, there are three main differences between their work and ours. First, unlike us, they do not assume that the given DAGs are defined over the same set of nodes. Second, unlike us, they assume that there exists a node ordering that is consistent with all the given DAGs. Third, their goal is to find a MDI map whereas ours is to find the MDI map that has the fewest parameters associated among all the MDI maps, i.e. the consensus DAG. Finally, \citet{NielsenandParsons2007} develop a general framework to construct the consensus DAG gradually. Their framework is general in the sense that it is not tailored to any particular definition of consensus DAG. Instead, it relies upon a score to be defined by the user and that each expert will use to score different extensions to the current partial consensus DAG. The individual scores are then combined to choose the extension to perform. Unfortunately, we do not see how this framework could be applied to our definition of consensus DAG. Specifically, we do not see how each expert could score the extensions independently of the other experts, what the score would look like, or how the scores would be combined. 

It is worth recalling that this paper deals with the combination of probability distributions expressed as BNs. Those readers interested in the combination of probability distributions expressed in non-graphical numerical forms are referred to, for instance, \citep{GenestandZidek1986}. Note also that we are interested in the combination before any data is observed. Those readers interested in the combination after some data has been observed and each expert has updated her beliefs accordingly are referred to, for instance, \citep{NgandAbramson1994}. Finally, note also that we aim at combining the given DAGs into a DAG, the consensus DAG. Those readers interested in finding not a DAG but graphical features (e.g. arcs or paths) all or a significant number of experts agree upon may want to consult \citep{FriedmanandKoller2003,Harteminketal.2002,Pennaetal.2004}, since these works deal with a similar problem.

The rest of the paper is organized as follows. We start by reviewing some preliminary concepts in Section \ref{sec:preliminaries}. We analyze the complexity of finding the consensus DAG in Section \ref{sec:np-hard}. We discuss the heuristic for finding an approximated consensus DAG in more detail in Section \ref{sec:heuristic}. We introduce Methods A and B in Section \ref{sec:methodsab} and show that they are not correct. We correct them in Section \ref{sec:correctness}. We analyze the complexity of the corrected Methods A and B in Section \ref{sec:efficiency} and show that they are more efficient than any other approach we can think of to solve the same problem. We close with some discussion in Section \ref{sec:discussion}.

\section{Preliminaries}\label{sec:preliminaries}

In this section, we review some concepts used in this paper. All the DAGs, probability distributions and independence models in this paper are defined over ${\mathbf V}$, unless otherwise stated. If $A \rightarrow B$ is in a DAG $G$, then we say that $A$ and $B$ are {\em adjacent} in $G$. Moreover, we say that $A$ is a {\em parent} of $B$ and $B$ a {\em child} of $A$ in $G$. We denote the parents of $B$ in $G$ by $Pa_G(B)$. A node is called a {\em sink node} in $G$ if it has no children in $G$. A {\em route} between two nodes $A$ and $B$ in $G$ is a sequence of nodes starting with $A$ and ending with $B$ such that every two consecutive nodes in the sequence are adjacent in $G$. Note that the nodes in a route are not necessarily distinct. The {\em length} of a route is the number of (not necessarily distinct) arcs in the route. We treat all the nodes in $G$ as routes of length zero. A route between $A$ and $B$ is called {\em descending} from $A$ to $B$ if all the arcs in the route are directed towards $B$. If there is a descending route from $A$ to $B$, then $B$ is called a {\em descendant} of $A$. Note that $A$ is a descendant of itself, since we allow routes of length zero. Given a subset ${\mathbf X} \subseteq {\mathbf V}$, a node $A \in {\mathbf X}$ is called {\em maximal} in $G$ if $A$ is not descendant of any node in ${\mathbf X} \setminus \{A\}$ in $G$. Given a route $\rho$ between $A$ and $B$ in $G$ and a route $\rho'$ between $B$ and $C$ in $G$, $\rho \cup \rho'$ denotes the route between $A$ and $C$ in $G$ resulting from appending $\rho'$ to $\rho$. 

The {\em number of parameters} associated with a DAG $G$ is $\sum_{B \in {\mathbf V}} [\prod_{A \in Pa_G(B)} r_A] (r_B - 1)$, where $r_A$ and $r_B$ are the numbers of states of the random variables corresponding to the node $A$ and $B$. An arc $A \rightarrow B$ in $G$ is said to be {\em covered} if $Pa_G(A)=Pa_G(B) \setminus \{A\}$. By {\em covering} an arc $A \rightarrow B$ in $G$ we mean adding to $G$ the smallest set of arcs so that $A \rightarrow B$ becomes covered. We say that a node $C$ is a {\em collider} in a route in a DAG if there exist two nodes $A$ and $B$ such that $A \rightarrow C \leftarrow B$ is a subroute of the route. Note that $A$ and $B$ may coincide. Let ${\mathbf X}$, ${\mathbf Y}$ and ${\mathbf Z}$ denote three disjoint subsets of ${\mathbf V}$. A route in a DAG is said to be {\em ${\mathbf Z}$-active} when (i) every collider node in the route is in ${\mathbf Z}$, and (ii) every non-collider node in the route is outside ${\mathbf Z}$. When there is no route in a DAG $G$ between a node in ${\mathbf X}$ and a node in ${\mathbf Y}$ that is ${\mathbf Z}$-active, we say that ${\mathbf X}$ is {\em separated} from ${\mathbf Y}$ given ${\mathbf Z}$ in $G$ and denote it as ${\mathbf X} \ci_G {\mathbf Y} | {\mathbf Z}$. We denote by ${\mathbf X} \nci_G {\mathbf Y} | {\mathbf Z}$ that ${\mathbf X} \ci_G {\mathbf Y} | {\mathbf Z}$ does not hold. This definition of separation is equivalent to other more common definitions \citep[Section 5.1]{Studeny1998}.

Let ${\mathbf X}$, ${\mathbf Y}$, ${\mathbf Z}$ and ${\mathbf W}$ denote four disjoint subsets of ${\mathbf V}$. Let us abbreviate ${\mathbf X} \cup {\mathbf Y}$ as ${\mathbf X} {\mathbf Y}$. An {\em independence model} $M$ is a set of statements of the form ${\mathbf X} \ci_M {\mathbf Y} | {\mathbf Z}$, meaning that ${\mathbf X}$ is independent of ${\mathbf Y}$ given ${\mathbf Z}$. Given a subset ${\mathbf U} \subseteq {\mathbf V}$, we denote by $[M]_{\mathbf U}$ all the statements in $M$ such that ${\mathbf X}, {\mathbf Y}, {\mathbf Z} \subseteq {\mathbf U}$. Given two independence models $M$ and $N$, we denote by $M \subseteq N$ that if ${\mathbf X} \ci_M {\mathbf Y} | {\mathbf Z}$ then ${\mathbf X} \ci_N {\mathbf Y} | {\mathbf Z}$. We say that $M$ is a {\em graphoid} if it satisfies the following properties: {\em symmetry} ${\mathbf X} \ci_M {\mathbf Y} | {\mathbf Z} \Rightarrow {\mathbf Y} \ci_M {\mathbf X} | {\mathbf Z}$, {\em decomposition} ${\mathbf X} \ci_M {\mathbf Y} {\mathbf W} | {\mathbf Z} \Rightarrow {\mathbf X} \ci_M {\mathbf Y} | {\mathbf Z}$, {\em weak union} ${\mathbf X} \ci_M {\mathbf Y} {\mathbf W} | {\mathbf Z} \Rightarrow {\mathbf X} \ci_M {\mathbf Y} | {\mathbf Z} {\mathbf W}$, {\em contraction} ${\mathbf X} \ci_M {\mathbf Y} | {\mathbf Z} {\mathbf W} \land {\mathbf X} \ci_M {\mathbf W} | {\mathbf Z} \Rightarrow {\mathbf X} \ci_M {\mathbf Y} {\mathbf W} | {\mathbf Z}$, and {\em intersection} ${\mathbf X} \ci_M {\mathbf Y} | {\mathbf Z} {\mathbf W} \land {\mathbf X} \ci_M {\mathbf W} | {\mathbf Z} {\mathbf Y} \Rightarrow {\mathbf X} \ci_M {\mathbf Y} {\mathbf W} | {\mathbf Z}$. The independence model {\em induced by a probability distribution $p$}, denoted as $I(p)$, is the set of probabilistic independences in $p$. The independence model {\em induced by a DAG $G$}, denoted as $I(G)$, is the set of separation statements ${\mathbf X} \ci_G {\mathbf Y} | {\mathbf Z}$. It is known that $I(G)$ is a graphoid \citep[Lemma 3.1]{StudenyandBouckaert1998}. Moreover, $I(G)$ satisfies the {\em composition} property ${\mathbf X} \ci_G {\mathbf Y} | {\mathbf Z} \land {\mathbf X} \ci_G {\mathbf W} | {\mathbf Z} \Rightarrow {\mathbf X} \ci_G {\mathbf Y} {\mathbf W} | {\mathbf Z}$ \citep[Proposition 1]{ChickeringandMeek2002}. Two DAGs $G$ and $H$ are called {\em equivalent} if $I(G)=I(H)$.

A DAG $G$ is a {\em directed independence map} of an independence model $M$ if $I(G) \subseteq M$. Moreover, $G$ is a {\em minimal} directed independence (MDI) map of $M$ if removing any arc from $G$ makes it cease to be a directed independence map of $M$. We say that $G$ and an ordering of its nodes are {\em consistent} when, for every arc $A \rightarrow B$ in $G$, $A$ precedes $B$ in the node ordering. We say that a DAG $G_{\alpha}$ is a MDI map of an independence model $M$ {\em relative to a node ordering $\alpha$} if $G_{\alpha}$ is a MDI map of $M$ and $G_{\alpha}$ is consistent with $\alpha$. If $M$ is a graphoid, then $G_{\alpha}$ is unique \citep[Theorems 4 and 9]{Pearl1988}. Specifically, for each node $A$, $Pa_{G_{\alpha}}(A)$ is the smallest subset ${\mathbf X}$ of the predecessors of $A$ in $\alpha$, $Pre_{\alpha}(A)$, such that $A \ci_M Pre_{\alpha}(A) \setminus {\mathbf X} | {\mathbf X}$.

\section{Finding a Consensus DAG is NP-Hard}\label{sec:np-hard}

Recall that we have defined the consensus DAG of a given set of DAGs $G^1, \ldots, G^m$ as the DAG that has the fewest parameters associated among all the MDI maps of $\cap_{i=1}^m I(G^i)$. A sensible way to start the quest for the consensus DAG is by investigating whether there can exist several non-equivalent consensus DAGs. The following theorem answers this question.

\begin{thm}
There exists a set of DAGs that has two non-equivalent consensus DAGs.
\end{thm}

\begin{proof}

Consider the following two DAGs over four random variables with the same number of states each:

\begin{table}[h]
\centering
\begin{tabular}{cc}
\begin{tabular}{ccc}
$I$ & $\leftarrow$ & $J$\\
$\downarrow$ & &\\
$K$ & $\rightarrow$ & $L$\\
\end{tabular}
& \hspace{2cm}
\begin{tabular}{ccc}
$I$ & $\rightarrow$ & $J$\\
& & $\downarrow$\\
$K$ & $\leftarrow$ & $L$\\
\end{tabular}
\end{tabular}
\end{table}

Any of the following two non-equivalent DAGs is the consensus DAG of the two DAGs above:

\begin{table}[h]
\centering
\begin{tabular}{cc}
\begin{tabular}{ccc}
$I$ & $\rightarrow$ & $J$\\
$\downarrow$ & $\searrow$ & $\uparrow$\\
$K$ & $\leftarrow$ & $L$\\
\end{tabular}
& \hspace{2cm}
\begin{tabular}{ccc}
$I$ & $\leftarrow$ & $J$\\
$\uparrow$ & $\nearrow$ & $\downarrow$\\
$K$ & $\rightarrow$ & $L$\\
\end{tabular}
\end{tabular}
\end{table}

\end{proof}

A natural follow-up question to investigate is whether a consensus DAG can be found efficiently. Unfortunately, finding a consensus DAG is NP-hard, as we prove below. Specifically, we prove that the following decision problem is NP-hard:

{\bf CONSENSUS}
\begin{itemize}
\item INSTANCE: A set of DAGs $G^1, \ldots, G^m$ over ${\mathbf V}$, and a positive integer $d$.
\item QUESTION: Does there exist a DAG $G$ over ${\mathbf V}$ such that $I(G) \subseteq \cap_{i=1}^m I(G^i)$ and the number of parameters associated with $G$ is not greater than $d$ ?
\end{itemize}

Proving that CONSENSUS is NP-hard implies that finding the consensus DAG is also NP-hard, because if there existed an efficient algorithm for finding the consensus DAG, then we could use it to solve CONSENSUS efficiently. Our proof makes use of the following two decision problems:

{\bf FEEDBACK ARC SET}
\begin{itemize}
\item INSTANCE: A directed graph $G=({\mathbf V},{\mathbf A})$ and a positive integer $k$.
\item QUESTION: Does there exist a subset ${\mathbf B} \subset {\mathbf A}$ such that $|{\mathbf B}| \leq k$ and ${\mathbf B}$ has at least one arc from every directed cycle in $G$ ?
\end{itemize}

{\bf LEARN}
\begin{itemize}
\item INSTANCE: A probability distribution $p$ over ${\mathbf V}$, and a positive integer $d$.
\item QUESTION: Does there exist a DAG $G$ over ${\mathbf V}$ such that $I(G) \subseteq I(p)$ and the number of parameters associated with $G$ is not greater than $d$ ?
\end{itemize}

FEEDBACK ARC SET is NP-complete \citep{GareyandJohnson1979}. FEEDBACK ARC SET remains NP-complete for directed graphs in which the total degree of each vertex is at most three \citep{Gavril1977}. This degree-bounded FEEDBACK ARC SET problem is used in \citep{Chickeringetal.2004} to prove that LEARN is NP-hard. In their proof, \citet{Chickeringetal.2004} use the following polynomial reduction of any instance of the degree-bounded FEEDBACK ARC SET into an instance of LEARN:
\begin{itemize}
\item Let the instance of the degree-bounded FEEDBACK ARC SET consist of the directed graph $F=({\mathbf V}^F,{\mathbf A}^F)$ and the positive integer $k$.
\item Let $L$ denote a DAG whose nodes and arcs are determined from $F$ as follows. For every arc $V^F_i \rightarrow V^F_j$ in ${\mathbf A}^F$, create the following nodes and arcs in $L$: 

\begin{table}[h]
\centering
\begin{tabular}{ccccccccc}
&& $A_{ij \:\: (9)}$ &&&& $D_{ij \:\: (9)}$ &&\\
&& $\downarrow$ &&&& $\downarrow$ &&\\
$V^F_{i \:\: (9)}$ & $\rightarrow$ & $B_{ij \:\: (2)}$ && $H_{ij \:\: (2)}$ && $E_{ij \:\: (2)}$ & $\leftarrow$ & $G_{ij \:\: (9)}$\\
&& $\downarrow$ & $\swarrow$ && $\searrow$ & $\downarrow$ &&\\
&& $C_{ij \:\: (3)}$ &&&& $F_{ij \:\: (2)}$ & $\rightarrow$ & $V^F_{j \:\: (9)}$\\
\end{tabular}
\end{table}

The number in parenthesis besides each node is the number of states of the corresponding random variable. Let ${\mathbf H}^L$ denote all the nodes $H_{ij}$ in $L$, and let ${\mathbf V}^L$ denote the rest of the nodes in $L$.
\item Specify a (join) probability distribution $p({\mathbf H}^L,{\mathbf V}^L)$ such that $I(p({\mathbf H}^L,{\mathbf V}^L)) = I(L)$. 
\item Let the instance of LEARN consist of the (marginal) probability distribution $p({\mathbf V}^L)$ and the positive integer $d$, where $d$ is computed from $F$ and $k$ as shown in \citep[Equation 2]{Chickeringetal.2004}.
\end{itemize}

We now describe how the instance of LEARN resulting from the reduction above can be further reduced into an instance of CONSENSUS in polynomial time:
\begin{itemize}
\item Let $C^1$ denote the DAG over ${\mathbf V}^L$ that has all and only the arcs in $L$ whose both endpoints are in ${\mathbf V}^L$.
\item Let $C^2$ denote the DAG over ${\mathbf V}^L$ that only has the arcs $B_{ij} \rightarrow C_{ij} \leftarrow F_{ij}$ for all $i$ and $j$.
\item Let $C^3$ denote the DAG over ${\mathbf V}^L$ that only has the arcs $C_{ij} \rightarrow F_{ij} \leftarrow E_{ij}$ for all $i$ and $j$.
\item Let the instance of CONSENSUS consist of the DAGs $C^1$, $C^2$ and $C^3$, and the positive integer $d$.
\end{itemize}

\begin{thm}\label{the:np}
CONSENSUS is NP-hard.
\end{thm}

\begin{proof}

We start by proving that there is a polynomial reduction of any instance ${\mathcal F}$ of the degree-bounded FEEDBACK ARC SET into an instance ${\mathcal C}$ of CONSENSUS. First, reduce ${\mathcal F}$ into an instance ${\mathcal L}$ of LEARN as shown in \citep{Chickeringetal.2004} and, then, reduce ${\mathcal L}$ into ${\mathcal C}$ as shown above.

We now prove that there is a solution to ${\mathcal F}$ iff there is a solution to ${\mathcal C}$. Theorems 8 and 9 in \citep{Chickeringetal.2004} prove that there is a solution to ${\mathcal F}$ iff there is a solution to ${\mathcal L}$. Therefore, it only remains to prove that there is a solution to ${\mathcal L}$ iff there is a solution to ${\mathcal C}$. Let $L$ and $p({\mathbf H}^L,{\mathbf V}^L)$ denote the DAG and the probability distribution constructed in the reduction of ${\mathcal F}$ into ${\mathcal L}$. Recall that $I(p({\mathbf H}^L,{\mathbf V}^L)) = I(L)$. Moreover:

\begin{itemize}
\item Let $L^1$ denote the DAG over $({\mathbf H}^L,{\mathbf V}^L)$ that has all and only the arcs in $L$ whose both endpoints are in ${\mathbf V}^L$.
\item Let $L^2$ denote the DAG over $({\mathbf H}^L,{\mathbf V}^L)$ that only has the arcs $B_{ij} \rightarrow C_{ij} \leftarrow H_{ij} \rightarrow F_{ij}$ for all $i$ and $j$.
\item Let $L^3$ denote the DAG over $({\mathbf H}^L,{\mathbf V}^L)$ that only has the arcs $C_{ij} \leftarrow H_{ij} \rightarrow F_{ij} \leftarrow E_{ij}$ for all $i$ and $j$.
\end{itemize}

Note that any separation statement that holds in $L$ also holds in $L^1$, $L^2$ and $L^3$. Then, $I(p({\mathbf H}^L,{\mathbf V}^L)) = I(L) \subseteq \cap_{i=1}^3 I(L^i)$ and, thus, $I(p({\mathbf V}^L)) \subseteq [\cap_{i=1}^3 I(L^i)]_{{\mathbf V}^L} = \cap_{i=1}^3 [I(L^i)]_{{\mathbf V}^L}$. Let $C^1$, $C^2$ and $C^3$ denote the DAGs constructed in the reduction of ${\mathcal L}$ into ${\mathcal C}$. Note that $[I(L^i)]_{{\mathbf V}^L} = I(C^i)$ for all $i$. Then, $I(p({\mathbf V}^L)) \subseteq \cap_{i=1}^3 I(C^i)$ and, thus, if there is a solution to ${\mathcal L}$ then there is a solution to ${\mathcal C}$. We now prove the opposite. The proof is essentially the same as that of \citep[Theorem 9]{Chickeringetal.2004}. Let us define the $(V_i, V_j)$ edge component of a DAG $G$ over ${\mathbf V}^L$ as the subgraph of $G$ that has all and only the arcs in $G$ whose both endpoints are in $\{V_i, A_{ij}, B_{ij}, C_{ij}, D_{ij}, E_{ij}, F_{ij}, G_{ij}, V_j\}$. Given a solution $C$ to ${\mathcal C}$, we create another solution $C'$ to ${\mathcal C}$ as follows:

\begin{itemize}
\item Initialize $C'$ to $C^1$.
\item For every $(V_i, V_j)$ edge component of $C$, if there is no directed path in $C$ from $V_i$ to $V_j$, then add to $C'$ the arcs $E_{ij} \rightarrow C_{ij} \leftarrow F_{ij}$.
\item For every $(V_i, V_j)$ edge component of $C$, if there is a directed path in $C$ from $V_i$ to $V_j$, then add to $C'$ the arcs $B_{ij} \rightarrow F_{ij} \leftarrow C_{ij}$.
\end{itemize}

Note that $C'$ is acyclic because $C$ is acyclic. Moreover, $I(C') \subseteq \cap_{i=1}^3 I(C^i)$ because $I(C') \subseteq I(C^i)$ for all $i$. In order to be able to conclude that $C'$ is a solution to ${\mathcal C}$, it only remains to prove that the number of parameters associated with $C'$ is not greater than $d$. Specifically, we prove below that $C'$ does not have more parameters associated than $C$, which has less than $d$ parameters associated because it is a solution to ${\mathcal C}$. 

As seen before, $I(C') \subseteq I(C^1)$. Likewise, $I(C) \subseteq I(C^1)$ because $C$ is a solution to ${\mathcal C}$. Thus, there exists a sequence $S$ (resp. $S'$) of covered arc reversals and arc additions that transforms $C^1$ into $C$ (resp. $C'$) \citep[Theorem 4]{Chickering2002}. Note that a covered arc reversal does not modify the number of parameters associated with a DAG, whereas an arc addition increases it \citep[Theorem 3]{Chickering1995}. Thus, $S$ and $S'$ monotonically increase the number of parameters associated with $C^1$ as they transform it. Recall that $C^1$ consists of a series of edge components of the form

\begin{table}[h]
\centering
\begin{tabular}{ccccccccc}
&& $A_{ij \:\: (9)}$ &&&& $D_{ij \:\: (9)}$ &&\\
&& $\downarrow$ &&&& $\downarrow$ &&\\
$V^F_{i \:\: (9)}$ & $\rightarrow$ & $B_{ij \:\: (2)}$ &&&& $E_{ij \:\: (2)}$ & $\leftarrow$ & $G_{ij \:\: (9)}$\\
&& $\downarrow$ &&&& $\downarrow$ &&\\
&& $C_{ij \:\: (3)}$ &&&& $F_{ij \:\: (2)}$ & $\rightarrow$ & $V^F_{j \:\: (9)}$\\
\end{tabular}
\end{table}
\newpage
The number in parenthesis besides each node is the number of states of the corresponding random variable. Let us study how the sequences $S$ and $S'$ modify each edge component of $C^1$. $S'$ simply adds the arcs $B_{ij} \rightarrow F_{ij} \leftarrow C_{ij}$ or the arcs $E_{ij} \rightarrow C_{ij} \leftarrow F_{ij}$. Note that adding the first pair of arcs results in an increase of 10 parameters, whereas adding the second pair of arcs results in an increase of 12 parameters. Unlike $S'$, $S$ may reverse some arc in the edge component. If that is the case, then $S$ must cover the arc first, which implies an increase of at least 16 parameters (covering $F_{ij} \rightarrow V_j$ by adding $E_{ij} \rightarrow V_j$ implies an increase of exactly 16 parameters, whereas any other arc covering implies a larger increase). Then, $S$ implies a larger increase in the number of parameters than $S'$. On the other hand, if $S$ does not reverse any arc in the edge component, then $S$ simply adds the arcs that are in $C$ but not in $C^1$. Note that either $C_{ij} \rightarrow F_{ij}$ or $C_{ij} \leftarrow F_{ij}$ is in $C$, because otherwise $C_{ij} \ci_C F_{ij} | {\mathbf Z}$ for some ${\mathbf Z} \subset {\mathbf V}^L$ which contradicts the fact that $C$ is a solution to ${\mathcal C}$ since $C_{ij} \nci_{C^2} F_{ij} | {\mathbf Z}$. If $C_{ij} \rightarrow F_{ij}$ is in $C$, then either $B_{ij} \rightarrow F_{ij}$ or $B_{ij} \leftarrow F_{ij}$ is in $C$ because otherwise $B_{ij} \ci_C F_{ij} | {\mathbf Z}$ for some ${\mathbf Z} \subset {\mathbf V}^L$ such that $C_{ij} \in {\mathbf Z}$, which contradicts the fact that $C$ is a solution to ${\mathcal C}$ since $B_{ij} \nci_{C^2} F_{ij} | {\mathbf Z}$. As $B_{ij} \leftarrow F_{ij}$ would create a cycle in $C$, $B_{ij} \rightarrow F_{ij}$ is in $C$. Therefore, $S$ adds the arcs $B_{ij} \rightarrow F_{ij} \leftarrow C_{ij}$ and, by construction of $C'$, $S'$ also adds them. Thus, $S$ implies an increase of at least as many parameters as $S'$. On the other hand, if $C_{ij} \leftarrow F_{ij}$ is in $C$, then either $C_{ij} \rightarrow E_{ij}$ or $C_{ij} \leftarrow E_{ij}$ is in $C$ because otherwise $C_{ij} \ci_C E_{ij} | {\mathbf Z}$ for some ${\mathbf Z} \subset {\mathbf V}^L$ such that $F_{ij} \in {\mathbf Z}$, which contradicts the fact that $C$ is a solution to ${\mathcal C}$ since $C_{ij} \nci_{C^3} E_{ij} | {\mathbf Z}$. As $C_{ij} \rightarrow E_{ij}$ would create a cycle in $C$, $C_{ij} \leftarrow E_{ij}$ is in $C$. Therefore, $S$ adds the arcs $E_{ij} \rightarrow C_{ij} \leftarrow F_{ij}$ and, by construction of $C'$, $S'$ adds either the arcs $E_{ij} \rightarrow C_{ij} \leftarrow F_{ij}$ or the arcs $B_{ij} \rightarrow F_{ij} \leftarrow C_{ij}$. In any case, $S$ implies an increase of at least as many parameters as $S'$. Consequently, $C'$ does not have more parameters associated than $C$.

Finally, note that $I(p({\mathbf V}^L)) \subseteq I(C')$ by \citep[Lemma 7]{Chickeringetal.2004}. Thus, if there is a solution to ${\mathcal C}$ then there is a solution to ${\mathcal L}$.

\end{proof}

It is worth noting that our proof above contains two restrictions. First, the number of DAGs to consensuate is three. Second, the number of states of each random variable in ${\mathbf V}^L$ is not arbitrary but prescribed. The first restriction is easy to relax: Our proof can be extended to consensuate more than three DAGs by simply letting $C^i$ be a DAG over ${\mathbf V}^L$ with no arcs for all $i > 3$. However, it is an open question whether CONSENSUS remains NP-hard when the number of DAGs to consensuate is two and/or the number of states of each random variable in ${\mathbf V}^L$ is arbitrary.

The following theorem strentghens the previous one.

\begin{thm}\label{the:np2}
CONSENSUS is NP-complete.
\end{thm}

\begin{proof}

By Theorem \ref{the:np}, all that remains to prove is that CONSENSUS is in NP, i.e. that we can verify in polynomial time if a given DAG $G$ is a solution to a given instance of CONSENSUS. 

Let $\alpha$ denote any node ordering that is consistent with $G$. The causal list of $G$ relative to $\alpha$ is the set of separation statements $A \ci_{G} Pre_{\alpha}(A) \setminus Pa_{G}(A) | Pa_{G}(A)$ for all node $A$. It is known that $I(G)$ coincides with the closure with respect to the graphoid properties of the causal list of $G$ relative to $\alpha$ \citep[Corollary 7]{Pearl1988}. Therefore, $I(G) \subseteq \cap_{i=1}^m I(G^i)$ iff $A \ci_{G^i} Pre_{\alpha}(A) \setminus Pa_{G}(A) | Pa_{G}(A)$ for all $1 \leq i \leq m$, because $\cap_{i=1}^m I(G^i)$ is a graphoid \citep[Corollary 1]{delSagradoandMoral2003}. Let $n$, $a$ and $a_i$ denote, respectively, the number of nodes in $G$, the number of arcs in $G$, and the number of arcs in $G_i$. Let $b=$ max$_{1 \leq i \leq m} \: a_i$. Checking a separation statement in $G_i$ takes $O(a_i)$ time \citep[p. 530]{Geigeretal.1990}. Then, checking whether $I(G) \subseteq \cap_{i=1}^m I(G^i)$ takes $O(m n b)$ time. Finally, note that computing the number of parameters associated with $G$ takes $O(a)$.

\end{proof}

\section{Finding an Approximated Consensus DAG}\label{sec:heuristic}

Since finding a consensus DAG of some given DAGs is NP-hard, we decide to resort to heuristics to find an approximated consensus DAG. This does not mean that we discard the existence of fast super-polynomial algorithms. It simply means that we do not pursue that possibility in this paper. Specifically, in this paper we consider the following heuristic due to \citet{MatzkevichandAbramson1992,MatzkevichandAbramson1993a,MatzkevichandAbramson1993b}. First, let $\alpha$ denote any ordering of the nodes in the given DAGs, which we denote here as $G^1, \ldots, G^m$. Then, find the MDI map $G^i_{\alpha}$ of each $G^i$ relative to $\alpha$. Finally, let the approximated consensus DAG be the DAG whose arcs are exactly the union of the arcs in $G^1_{\alpha}, \ldots, G^m_{\alpha}$. The following theorem justifies taking the union of the arcs. Specifically, it proves that the DAG returned by the heuristic is the consensus DAG if this was required to be consistent with $\alpha$.

\begin{thm}\label{the:heuristic}
The DAG $H$ returned by the heuristic above is the DAG that has the fewest parameters associated among all the MDI maps of $\cap_{i=1}^m I(G^i)$ relative to $\alpha$.
\end{thm}

\begin{proof}

We start by proving that $H$ is a MDI map of $\cap_{i=1}^m I(G^i)$. First, we show that $I(H) \subseteq \cap_{i=1}^m I(G^i)$. It suffices to note that $I(H) \subseteq \cap_{i=1}^m I(G^i_{\alpha})$ because each $G^i_{\alpha}$ is a subgraph of $H$, and that $\cap_{i=1}^m I(G^i_{\alpha}) \subseteq \cap_{i=1}^m I(G^i)$ because $I(G^i_{\alpha}) \subseteq I(G^i)$ for all $i$. Now, assume to the contrary that the DAG $H'$ resulting from removing an arc $A \rightarrow B$ from $H$ satisfies that $I(H') \subseteq \cap_{i=1}^m I(G^i)$. By construction of $H$, $A \rightarrow B$ is in $G^i_{\alpha}$ for some $i$, say $i=j$. Note that $B \ci_{H'} Pre_{\alpha}(B) \setminus Pa_{H'}(B) | Pa_{H'}(B)$, which implies $B \ci_{G^j} Pre_{\alpha}(B) \setminus ((\cup_{i=1}^m Pa_{G^i_{\alpha}}(B)) \setminus \{A\}) | (\cup_{i=1}^m Pa_{G^i_{\alpha}}(B)) \setminus \{A\}$ because $Pa_{H'}(B)=(\cup_{i=1}^m Pa_{G^i_{\alpha}}(B)) \setminus \{A\}$ and $I(H') \subseteq \cap_{i=1}^m I(G^i)$. Note also that $B \ci_{G^j_{\alpha}} Pre_{\alpha}(B) \setminus Pa_{G^j_{\alpha}}(B) | Pa_{G^j_{\alpha}}(B)$, which implies $B \ci_{G^j} Pre_{\alpha}(B) \setminus Pa_{G^j_{\alpha}}(B) | Pa_{G^j_{\alpha}}(B)$ because $I(G^j_{\alpha}) \subseteq I(G^j)$. Therefore, $B \ci_{G^j} Pre_{\alpha}(B) \setminus (Pa_{G^j_{\alpha}}(B) \setminus \{A\}) | Pa_{G^j_{\alpha}}(B) \setminus \{A\}$ by intersection. However, this contradicts the fact that $G^j_{\alpha}$ is the MDI map of $G^j$ relative to $\alpha$. Then, $H$ is a MDI map of $\cap_{i=1}^m I(G^i)$ relative to $\alpha$.

Finally, note that $\cap_{i=1}^m I(G^i)$ is a graphoid \citep[Corollary 1]{delSagradoandMoral2003}. Consequently, $H$ is the only MDI map of $\cap_{i=1}^m I(G^i)$ relative to $\alpha$.

\end{proof}

A key step in the heuristic above is, of course, choosing a good node ordering $\alpha$. Unfortunately, the fact that CONSENSUS is NP-hard implies that it is also NP-hard to find the best node ordering $\alpha$, i.e. the node ordering that makes the heuristic to return the MDI map of $\cap_{i=1}^m I(G^i)$ that has the fewest parameters associated. To see it, note that if there existed an efficient algorithm for finding the best node ordering, then Theorem \ref{the:heuristic} would imply that we could solve CONSENSUS efficiently by running the heuristic with the best node ordering. 

In the last sentence, we have implicitly assumed that the heuristic is efficient, which implies that we have implicitly assumed that we can efficiently find the MDI map $G^i_{\alpha}$ of each $G^i$. The rest of this paper shows that this assumption is correct.

\section{Methods A and B are not Correct}\label{sec:methodsab}

\begin{figure}[t]
\centering
\small
\begin{tabular}{rl}
\hline
\\
& \underline{Construct $\beta$($G$, $\alpha$)}\\
\\
& /* Given a DAG $G$ and a node ordering $\alpha$, the algorithm returns a node ordering $\beta$ that\\
& is consistent with $G$ and as close to $\alpha$ as possible */\\
\\
1 & $\beta=\emptyset$\\
2 & $G'=G$\\
3 & Let $A$ denote a sink node in $G'$\\
/* 3 & Let $A$ denote the rightmost node in $\alpha$ that is a sink node in $G'$ */\\
4 & Add $A$ as the leftmost node in $\beta$\\
5 & Let $B$ denote the right neighbor of $A$ in $\beta$\\
6 & If $B \neq \emptyset$ and $A \notin Pa_G(B)$ and $A$ is to the right of $B$ in $\alpha$ then\\
7 & \hspace{0.2cm} Interchange $A$ and $B$ in $\beta$\\
8 & \hspace{0.2cm} Go to line 5\\
9 & Remove $A$ and all its incoming arcs from $G'$\\
10 & If $G' \neq \emptyset$ then go to line 3\\
11 & Return $\beta$\\
\\
& \underline{Method A($G$, $\alpha$)}\\
\\
& /* Given a DAG $G$ and a node ordering $\alpha$, the algorithm returns $G_{\alpha}$ */\\
\\
1 & $\beta$=Construct $\beta$($G$, $\alpha$)\\
2 & Let $Y$ denote the leftmost node in $\beta$ whose left neighbor in $\beta$ is to its right in $\alpha$\\
3 & Let $Z$ denote the left neighbor of $Y$ in $\beta$\\
4 & If $Z$ is to the right of $Y$ in $\alpha$ then\\
5 & \hspace{0.2cm} If $Z \rightarrow Y$ is in $G$ then cover and reverse $Z \rightarrow Y$ in $G$\\
6 & \hspace{0.2cm} Interchange $Y$ and $Z$ in $\beta$\\
7 & \hspace{0.2cm} Go to line 3\\
8 & If $\beta \neq \alpha$ then go to line 2\\
9 & Return $G$\\
\\
& \underline{Method B($G$, $\alpha$)}\\
\\
& /* Given a DAG $G$ and a node ordering $\alpha$, the algorithm returns $G_{\alpha}$ */\\
\\
1 & $\beta$=Construct $\beta$($G$, $\alpha$)\\
2 & Let $Y$ denote the leftmost node in $\beta$ whose right neighbor in $\beta$ is to its left in $\alpha$\\
3 & Let $Z$ denote the right neighbor of $Y$ in $\beta$\\
4 & If $Z$ is to the left of $Y$ in $\alpha$ then\\
5 & \hspace{0.2cm} If $Y \rightarrow Z$ is in $G$ then cover and reverse $Y \rightarrow Z$ in $G$\\
6 & \hspace{0.2cm} Interchange $Y$ and $Z$ in $\beta$\\
7 & \hspace{0.2cm} Go to line 3\\
8 & If $\beta \neq \alpha$ then go to line 2\\
9 & Return $G$\\
\\
\hline
\\
\end{tabular}
\caption{Construct $\beta$, and Methods A and B. Our correction of Construct $\beta$ consists in replacing line 3 with the line in comments under it.}\label{fig:methodsab}
\end{figure}

\citet{MatzkevichandAbramson1993b} do not only propose the heuristic discussed in the previous section, but they also present two algorithms, called Methods A and B, for efficiently deriving the MDI map $G_{\alpha}$ of a DAG $G$ relative to a node ordering $\alpha$. The algorithms work iteratively by covering and reversing an arc in $G$ until the resulting DAG is consistent with $\alpha$. It is obvious that such a way of working produces a directed independence map of $G$. However, in order to arrive at $G_{\alpha}$, the arc to cover and reverse in each iteration must be carefully chosen. The pseudocode of Methods A and B can be seen in Figure \ref{fig:methodsab}. Method A starts by calling Construct $\beta$ to derive a node ordering $\beta$ that is consistent with $G$ and as close to $\alpha$ as possible (line 6). By $\beta$ being as close to $\alpha$ as possible, we mean that the number of arcs Methods A and B will later cover and reverse is kept at a minimum, because Methods A and B will use $\beta$ to choose the arc to cover and reverse in each iteration. In particular, Method A finds the leftmost node in $\beta$ that should be interchanged with its left neighbor (line 2) and it repeatedly interchanges this node with its left neighbor (lines 3-4 and 6-7). Each of these interchanges is preceded by covering and reversing the corresponding arc in $G$ (line 5). Method B is essentially identical to Method A. The only differences between them are that the word "right" is replaced by the word "left" and vice versa in lines 2-4, and that the arcs point in opposite directions in line 5.

\begin{figure}[t]
\centering
\begin{tabular}{c}
\hline
\\
\includegraphics[scale=0.4]{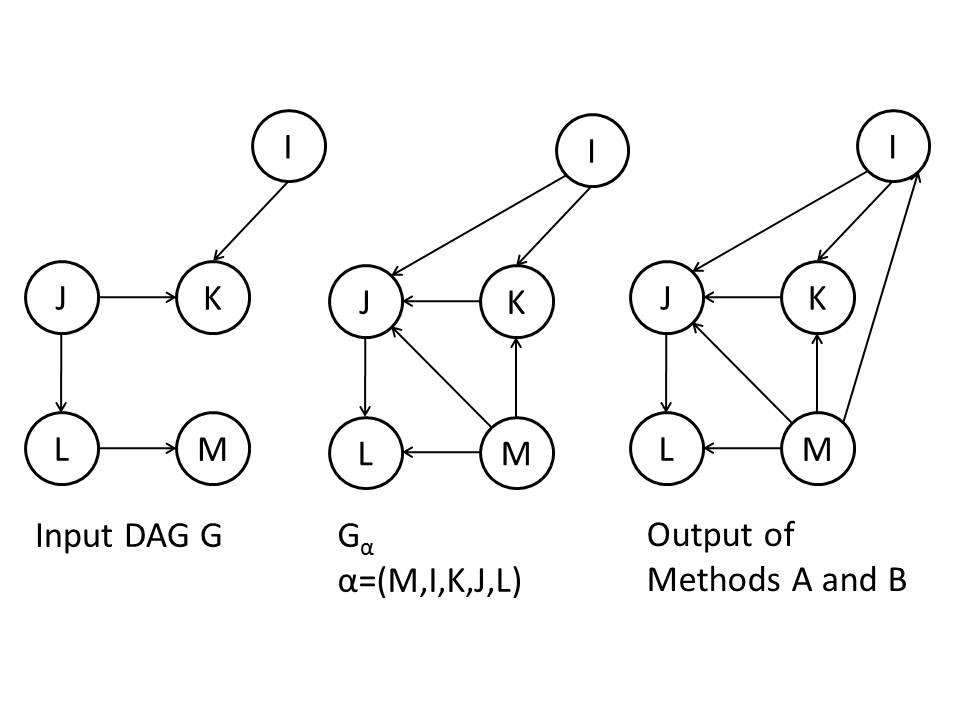}
\\
\hline
\\
\end{tabular}
\caption{A counterexample to the correctness of Methods A and B.}\label{fig:counter}
\end{figure}

Methods A and B are claimed to be correct in \citep[Theorem 4 and Corollary 2]{MatzkevichandAbramson1993b} although no proof is provided (a proof is just sketched). The following counterexample shows that Methods A and B are actually not correct. Let $G$ be the DAG in the left-hand side of Figure \ref{fig:counter}. Let $\alpha=(M,I,K,J,L)$. Then, we can make use of the characterization introduced in Section \ref{sec:preliminaries} to see that $G_{\alpha}$ is the DAG in the center of Figure \ref{fig:counter}. However, Methods A and B return the DAG in the right-hand side of Figure \ref{fig:counter}. To see it, we follow the execution of Methods A and B step by step. First, Methods A and B construct $\beta$ by calling Construct $\beta$, which runs as follows:
\begin{enumerate}
\item Initially, $\beta=\emptyset$ and $G'=G$.
\item Select the sink node $M$ in $G'$. Then, $\beta=(M)$. Remove $M$ and its incoming arcs from $G'$.
\item Select the sink node $L$ in $G'$. Then, $\beta=(L,M)$. No interchange in $\beta$ is performed because $L \in Pa_{G}(M)$. Remove $L$ and its incoming arcs from $G'$.
\item Select the sink node $K$ in $G'$. Then, $\beta=(K,L,M)$. No interchange in $\beta$ is performed because $K$ is to the left of $L$ in $\alpha$. Remove $K$ and its incoming arcs from $G'$.
\item Select the sink node $J$ in $G'$. Then, $\beta=(J,K,L,M)$. No interchange in $\beta$ is performed because $J \in Pa_{G}(K)$.
\item Select the sink node $I$ in $G'$. Then, $\beta=(I,J,K,L,M)$. No interchange in $\beta$ is performed because $I$ is to the left of $J$ in $\alpha$.
\end{enumerate}

When Construct $\beta$ ends, Methods A and B continue as follows:
\begin{enumerate}
\setcounter{enumi}{6}
\item Initially, $\beta=(I,J,K,L,M)$.
\item Add the arc $I \rightarrow J$ and reverse the arc $J \rightarrow K$ in $G$. Interchange $J$ and $K$ in $\beta$. Then, $\beta=(I,K,J,L,M)$.
\item Add the arc $J \rightarrow M$ and reverse the arc $L \rightarrow M$ in $G$. Interchange $L$ and $M$ in $\beta$. Then, $\beta=(I,K,J,M,L)$.
\item Add the arcs $I \rightarrow M$ and $K \rightarrow M$, and reverse the arc $J \rightarrow M$ in $G$. Interchange $J$ and $M$ in $\beta$. Then, $\beta=(I,K,M,J,L)$.
\item Reverse the arc $K \rightarrow M$ in $G$. Interchange $K$ and $M$ in $\beta$. Then, $\beta=(I,M,K,J,L)$.
\item Reverse the arc $I \rightarrow M$ in $G$. Interchange $I$ and $M$ in $\beta$. Then, $\beta=(M,I,K,J,L)=\alpha$.
\end{enumerate}

As a matter of fact, one can see as early as in step (8) above that Methods A and B will fail: One can see that $I$ and $M$ are not separated in the DAG resulting from step (8), which implies that $I$ and $M$ will not be separated in the DAG returned by Methods A and B, because covering and reversing arcs never introduces new separation statements. However, $I$ and $M$ are separated in $G_{\alpha}$.

Note that we constructed $\beta$ by selecting first $M$, then $L$, then $K$, then $J$, and finally $I$. However, we could have selected first $K$, then $I$, then $M$, then $L$, and finally $J$, which would have resulted in $\beta=(J,L,M,I,K)$. With this $\beta$, Methods A and B return $G_{\alpha}$. Therefore, it makes a difference which sink node is selected in line 3 of Construct $\beta$. However, Construct $\beta$ overlooks this detail. We propose correcting Construct $\beta$ by replacing line 3 by "Let $A$ denote the rightmost node in $\alpha$ that is a sink node in $G'$". Hereinafter, we assume that any call to Construct $\beta$ is a call to the corrected version thereof. The rest of this paper is devoted to prove that Methods A and B now do return $G_{\alpha}$.

\section{The Corrected Methods A and B are Correct}\label{sec:correctness}

\begin{figure}[t]
\centering
\small
\begin{tabular}{rl}
\hline
\\
& \underline{Method A2($G$, $\alpha$)}\\
\\
& /* Given a DAG $G$ and a node ordering $\alpha$, the algorithm returns $G_{\alpha}$ */\\
\\
1 & $\beta$=Construct $\beta$($G$, $\alpha$)\\
2 & Let $Y$ denote the leftmost node in $\beta$ that has not been considered before\\
3 & Let $Z$ denote the left neighbor of $Y$ in $\beta$\\
4 & If $Z \neq \emptyset$ and $Z$ is to the right of $Y$ in $\alpha$ then\\
5 & \hspace{0.2cm} If $Z \rightarrow Y$ is in $G$ then cover and reverse $Z \rightarrow Y$ in $G$\\
6 & \hspace{0.2cm} Interchange $Y$ and $Z$ in $\beta$\\
7 & \hspace{0.2cm} Go to line 3\\
8 & If $\beta \neq \alpha$ then go to line 2\\
9 & Return $G$\\
\\
& \underline{Method B2($G$, $\alpha$)}\\
\\
& /* Given a DAG $G$ and a node ordering $\alpha$, the algorithm returns $G_{\alpha}$ */\\
\\
1 & $\beta$=Construct $\beta$($G$, $\alpha$)\\
2 & Let $Y$ denote the rightmost node in $\alpha$ that has not been considered before\\
3 & Let $Z$ denote the right neighbor of $Y$ in $\beta$\\
4 & If $Z \neq \emptyset$ and $Z$ is to the left of $Y$ in $\alpha$ then\\
5 & \hspace{0.2cm} If $Y \rightarrow Z$ is in $G$ then cover and reverse $Y \rightarrow Z$ in $G$\\
6 & \hspace{0.2cm} Interchange $Y$ and $Z$ in $\beta$\\
7 & \hspace{0.2cm} Go to line 3\\
8 & If $\beta \neq \alpha$ then go to line 2\\
9 & Return $G$\\
\\
\hline
\\
\end{tabular}
\caption{Methods A2 and B2.}\label{fig:methodsa2b2b3}
\end{figure}

Before proving that Methods A and B are correct, we introduce some auxiliary lemmas. Their proof can be found in the appendix. Let us call {\em percolating $Y$ right-to-left in $\beta$} to iterating through lines 3-7 in Method A while possible. Let us modify Method A by replacing line 2 by "Let $Y$ denote the leftmost node in $\beta$ that has not been considered before" and by adding the check $Z \neq \emptyset$ to line 4. The pseudocode of the resulting algorithm, which we call Method A2, can be seen in Figure \ref{fig:methodsa2b2b3}. Method A2 percolates right-to-left in $\beta$ one by one all the nodes in the order in which they appear in $\beta$.

\begin{lem}\label{lem:a=a2}
Method A($G$, $\alpha$) and Method A2($G$, $\alpha$) return the same DAG.
\end{lem}

\begin{lem}\label{lem:a2=b}
Method A2($G$, $\alpha$) and Method B($G$, $\alpha$) return the same DAG.
\end{lem}

Let us call {\em percolating $Y$ left-to-right in $\beta$} to iterating through lines 3-7 in Method B while possible. Let us modify Method B by replacing line 2 by "Let $Y$ denote the rightmost node in $\alpha$ that has not been considered before" and by adding the check $Z \neq \emptyset$ to line 4. The pseudocode of the resulting algorithm, which we call Method B2, can be seen in Figure \ref{fig:methodsa2b2b3}. Method B2 percolates left-to-right in $\beta$ one by one all the nodes in the reverse order in which they appear in $\alpha$. 

\begin{lem}\label{lem:b=b2}
Method B($G$, $\alpha$) and Method B2($G$, $\alpha$) return the same DAG.
\end{lem}

\begin{figure}[t]
\centering
\begin{tabular}{c}
\hline
\\
\includegraphics[scale=0.6]{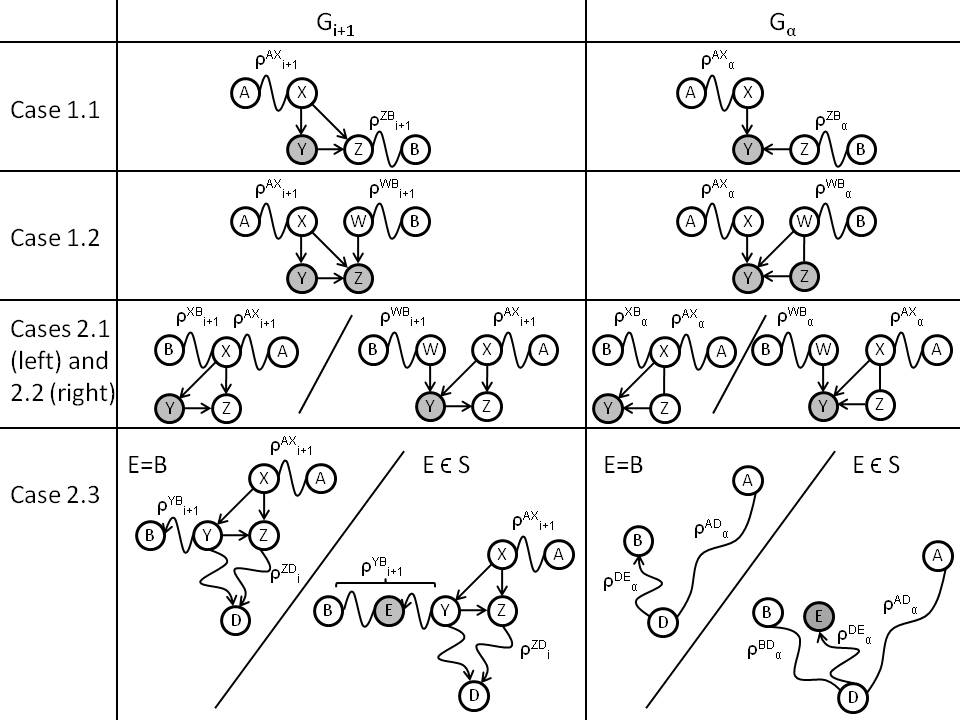}
\\
\\
\hline
\\
\end{tabular}
\caption{Different cases in the proof of Theorem \ref{the:correctness}. Only the relevant subgraphs of $G_{i+1}$ and $G_{\alpha}$ are depicted. An undirected edge between two nodes denotes that the nodes are adjacent. A curved edge between two nodes denotes an ${\mathbf S}$-active route between the two nodes. If the curved edge is directed, then the route is descending. A grey node denotes a node that is in ${\mathbf S}$.}\label{fig:cases}
\end{figure}

We are now ready to prove the main result of this paper.

\begin{thm}\label{the:correctness}
Let $G_{\alpha}$ denote the MDI map of a DAG $G$ relative to a node ordering $\alpha$. Then, Method A($G$, $\alpha$) and Method B($G$, $\alpha$) return $G_{\alpha}$.
\end{thm}

\begin{proof}

By Lemmas \ref{lem:a=a2}-\ref{lem:b=b2}, it suffices to prove that Method B2($G$, $\alpha$) returns $G_{\alpha}$. It is evident that Method B2 transforms $\beta$ into $\alpha$ and, thus, that it halts at some point. Therefore, Method B2 performs a finite sequence of $n$ modifications (arc additions and covered arc reversals) to $G$. Let $G_i$ denote the DAG resulting from the first $i$ modifications to $G$, and let $G_0=G$. Specifically, Method B2 constructs $G_{i+1}$ from $G_i$ by either (i) reversing the covered arc $Y \rightarrow Z$, or (ii) adding the arc $X \rightarrow Z$ for some $X \in Pa_{G_i}(Y) \setminus Pa_{G_i}(Z)$, or (iii) adding the arc $X \rightarrow Y$ for some $X \in Pa_{G_i}(Z) \setminus Pa_{G_i}(Y)$. Note that $I(G_{i+1}) \subseteq I(G_{i})$ for all $0 \leq i < n$ and, thus, that $I(G_{n}) \subseteq I(G_{0})$.

We start by proving that $G_i$ is a DAG that is consistent with $\beta$ for all $0 \leq i \leq n$. Since this is true for $G_0$ due to line 1, it suffices to prove that if $G_i$ is a DAG that is consistent with $\beta$ then so is $G_{i+1}$ for all $0 \leq i < n$. We consider the following four cases.

\begin{description}

\item[Case 1] Method B2 constructs $G_{i+1}$ from $G_i$ by reversing the covered arc $Y \rightarrow Z$. Then, $G_{i+1}$ is a DAG because reversing a covered arc does not create any cycle \citep[Lemma 1]{Chickering1995}. Moreover, note that $Y$ and $Z$ are interchanged in $\beta$ immediately after the covered arc reversal. Thus, $G_{i+1}$ is consistent with $\beta$. 

\item[Case 2] Method B2 constructs $G_{i+1}$ from $G_i$ by adding the arc $X \rightarrow Z$ for some $X \in Pa_{G_i}(Y) \setminus Pa_{G_i}(Z)$. Note that $X$ is to the left of $Y$ and $Y$ to the left of $Z$ in $\beta$, because $G_i$ is consistent with $\beta$. Then, $X$ is to the left of $Z$ in $\beta$ and, thus, $G_{i+1}$ is a DAG that is consistent with $\beta$.

\item[Case 3] Method B2 constructs $G_{i+1}$ from $G_i$ by adding the arc $X \rightarrow Y$ for some $X \in Pa_{G_i}(Z) \setminus Pa_{G_i}(Y)$. Note that $X$ is to the left of $Z$ in $\beta$ because $G_i$ is consistent with $\beta$, and $Y$ is the left neighbor of $Z$ in $\beta$ (recall line 3). Then, $X$ is to the left of $Y$ in $\beta$ and, thus, $G_{i+1}$ is a DAG that is consistent with $\beta$.

\item[Case 4] Note that $\beta$ may get modified before Method B2 constructs $G_{i+1}$ from $G_i$. Specifically, this happens when Method B2 executes lines 5-6 but there is no arc between $Y$ and $Z$ in $G_i$. However, the fact that $G_i$ is consistent with $\beta$ before $Y$ and $Z$ are interchanged in $\beta$ and the fact that $Y$ and $Z$ are neighbors in $\beta$ (recall line 3) imply that $G_i$ is consistent with $\beta$ after $Y$ and $Z$ have been interchanged.

\end{description}

Since Method B2 transforms $\beta$ into $\alpha$, it follows from the result proven above that $G_n$ is a DAG that is consistent with $\alpha$. In order to prove the theorem, i.e. that $G_n=G_{\alpha}$, all that remains to prove is that $I(G_{\alpha}) \subseteq I(G_n)$. To see it, note that $G_n=G_{\alpha}$ follows from $I(G_{\alpha}) \subseteq I(G_n)$, $I(G_n) \subseteq I(G_0)$, the fact that $G_n$ is a DAG that is consistent with $\alpha$, and the fact that $G_{\alpha}$ is the unique MDI map of $G_0$ relative to $\alpha$. Recall that $G_{\alpha}$ is guaranteed to be unique because $I(G_0)$ is a graphoid.

The rest of the proof is devoted to prove that $I(G_{\alpha}) \subseteq I(G_n)$. Specifically, we prove that if $I(G_{\alpha}) \subseteq I(G_i)$ then $I(G_{\alpha}) \subseteq I(G_{i+1})$ for all $0 \leq i < n$. Note that this implies that $I(G_{\alpha}) \subseteq I(G_n)$ because $I(G_{\alpha}) \subseteq I(G_0)$ by definition of MDI map. First, we prove it when Method B2 constructs $G_{i+1}$ from $G_i$ by reversing the covered arc $Y \rightarrow Z$. That the arc reversed is covered implies that $I(G_{i+1})=I(G_i)$ \citep[Lemma 1]{Chickering1995}. Thus, $I(G_{\alpha}) \subseteq I(G_{i+1})$ because $I(G_{\alpha}) \subseteq I(G_i)$.

Now, we prove that if $I(G_{\alpha}) \subseteq I(G_i)$ then $I(G_{\alpha}) \subseteq I(G_{i+1})$ for all $0 \leq i < n$ when Method B2 constructs $G_{i+1}$ from $G_i$ by adding an arc. Specifically, we prove that if there is an ${\mathbf S}$-active route $\rho^{AB}_{i+1}$ between two nodes $A$ and $B$ in $G_{i+1}$, then there is an ${\mathbf S}$-active route between $A$ and $B$ in $G_{\alpha}$. We prove this result by induction on the number of occurrences of the added arc in $\rho^{AB}_{i+1}$. We assume without loss of generality that the added arc occurs in $\rho^{AB}_{i+1}$ as few or fewer times than in any other ${\mathbf S}$-active route between $A$ and $B$ in $G_{i+1}$. We call this the minimality property of $\rho^{AB}_{i+1}$.\footnote{It is not difficult to show that the number of occurrences of the added arc in $\rho^{AB}_{i+1}$ is then at most two (see Case 2.1 for some intuition). However, the proof of the theorem is simpler if we ignore this fact.} If the number of occurrences of the added arc in $\rho^{AB}_{i+1}$ is zero, then $\rho^{AB}_{i+1}$ is an ${\mathbf S}$-active route between $A$ and $B$ in $G_i$ too and, thus, there is an ${\mathbf S}$-active route between $A$ and $B$ in $G_{\alpha}$ since $I(G_{\alpha}) \subseteq I(G_i)$. Assume as induction hypothesis that the result holds for up to $k$ occurrences of the added arc in $\rho^{AB}_{i+1}$. We now prove it for $k+1$ occurrences. We consider the following two cases. Each case is illustrated in Figure \ref{fig:cases}.

\begin{description}

\item[Case 1] Method B2 constructs $G_{i+1}$ from $G_i$ by adding the arc $X \rightarrow Z$ for some $X \in Pa_{G_i}(Y) \setminus Pa_{G_i}(Z)$. Note that $X \rightarrow Z$ occurs in $\rho^{AB}_{i+1}$.\footnote{Note that maybe $A=X$ and/or $B=Z$.} Let $\rho^{AB}_{i+1} = \rho^{AX}_{i+1} \cup X \rightarrow Z \cup \rho^{ZB}_{i+1}$. Note that $X \notin {\mathbf S}$ and $\rho^{AX}_{i+1}$ is ${\mathbf S}$-active in $G_{i+1}$ because, otherwise, $\rho^{AB}_{i+1}$ would not be ${\mathbf S}$-active in $G_{i+1}$. Then, there is an ${\mathbf S}$-active route $\rho^{AX}_{\alpha}$ between $A$ and $X$ in $G_{\alpha}$ by the induction hypothesis. Moreover, $Y \in {\mathbf S}$ because, otherwise, $\rho^{AX}_{i+1} \cup X \rightarrow Y \rightarrow Z \cup \rho^{ZB}_{i+1}$ would be an ${\mathbf S}$-active route between $A$ and $B$ in $G_{i+1}$ that would violate the minimality property of $\rho^{AB}_{i+1}$. Note that $Y \leftarrow Z$ is in $G_{\alpha}$ because (i) $Y$ and $Z$ are adjacent in $G_{\alpha}$ since $I(G_{\alpha}) \subseteq I(G_i)$, and (ii) $Z$ is to the left of $Y$ in $\alpha$ (recall line 4). Note also that $X \rightarrow Y$ is in $G_{\alpha}$. To see it, note that $X$ and $Y$ are adjacent in $G_{\alpha}$ since $I(G_{\alpha}) \subseteq I(G_i)$. Recall that Method B2 percolates left-to-right in $\beta$ one by one all the nodes in the reverse order in which they appear in $\alpha$. Method B2 is currently percolating $Y$ and, thus, the nodes to the right of $Y$ in $\alpha$ are to right of $Y$ in $\beta$ too. If $X \leftarrow Y$ were in $G_{\alpha}$ then $X$ would be to the right of $Y$ in $\alpha$ and, thus, $X$ would be to the right of $Y$ in $\beta$. However, this would contradict the fact that $X$ is to the left of $Y$ in $\beta$, which follows from the fact that $G_i$ is consistent with $\beta$. Thus, $X \rightarrow Y$ is in $G_{\alpha}$. We now consider two cases.

\begin{description}

\item[Case 1.1] Assume that $Z \notin {\mathbf S}$. Then, $\rho^{ZB}_{i+1}$ is ${\mathbf S}$-active in $G_{i+1}$ because, otherwise, $\rho^{AB}_{i+1}$ would not be ${\mathbf S}$-active in $G_{i+1}$. Then, there is an ${\mathbf S}$-active route $\rho^{ZB}_{\alpha}$ between $Z$ and $B$ in $G_{\alpha}$ by the induction hypothesis. Then, $\rho^{AX}_{\alpha} \cup X \rightarrow Y \leftarrow Z \cup \rho^{ZB}_{\alpha}$ is an ${\mathbf S}$-active route between $A$ and $B$ in $G_{\alpha}$.

\item[Case 1.2] Assume that $Z \in {\mathbf S}$. Then, $\rho^{ZB}_{i+1} = Z \leftarrow W \cup \rho^{WB}_{i+1}$.\footnote{Note that maybe $W=B$. Note also that $W \neq X$ because, otherwise, $\rho^{AX}_{i+1} \cup X \rightarrow Y \leftarrow X \cup \rho^{WB}_{i+1}$ would be an ${\mathbf S}$-active route between $A$ and $B$ in $G_{i+1}$ that would violate the minimality property of $\rho^{AB}_{i+1}$.} Note that $W \notin {\mathbf S}$ and $\rho^{WB}_{i+1}$ is ${\mathbf S}$-active in $G_{i+1}$ because, otherwise, $\rho^{AB}_{i+1}$ would not be ${\mathbf S}$-active in $G_{i+1}$. Then, there is an ${\mathbf S}$-active route $\rho^{WB}_{\alpha}$ between $W$ and $B$ in $G_{\alpha}$ by the induction hypothesis. Note that $W$ and $Z$ are adjacent in $G_{\alpha}$ since $I(G_{\alpha}) \subseteq I(G_i)$. This and the fact proven above that $Y \leftarrow Z$ is in $G_{\alpha}$ imply that $Y$ and $W$ are adjacent in $G_{\alpha}$ because, otherwise, $Y \nci_{G_i} W | {\mathbf U}$ but $Y \ci_{G_{\alpha}} W | {\mathbf U}$ for some ${\mathbf U} \subseteq {\mathbf V}$ such that $Z \in {\mathbf U}$, which would contradict that $I(G_{\alpha}) \subseteq I(G_i)$. In fact, $Y \leftarrow W$ is in $G_{\alpha}$. To see it, recall that the nodes to the right of $Y$ in $\alpha$ are to right of $Y$ in $\beta$ too. If $Y \rightarrow W$ were in $G_{\alpha}$ then $W$ would be to the right of $Y$ in $\alpha$ and, thus, $W$ would be to the right of $Y$ in $\beta$ too. However, this would contradict the fact that $W$ is to the left of $Y$ in $\beta$, which follows from the fact that $W$ is to the left of $Z$ in $\beta$ because $G_i$ is consistent with $\beta$, and the fact that $Y$ is the left neighbor of $Z$ in $\beta$ (recall line 3). Thus, $Y \leftarrow W$ is in $G_{\alpha}$. Then, $\rho^{AX}_{\alpha} \cup X \rightarrow Y \leftarrow W \cup \rho^{WB}_{\alpha}$ is an ${\mathbf S}$-active route between $A$ and $B$ in $G_{\alpha}$. 

\end{description}

\item[Case 2] Method B2 constructs $G_{i+1}$ from $G_i$ by adding the arc $X \rightarrow Y$ for some $X \in Pa_{G_i}(Z) \setminus Pa_{G_i}(Y)$. Note that $X \rightarrow Y$ occurs in $\rho^{AB}_{i+1}$.\footnote{Note that maybe $A=X$ and/or $B=Y$.} Let $\rho^{AB}_{i+1} = \rho^{AX}_{i+1} \cup X \rightarrow Y \cup \rho^{YB}_{i+1}$. Note that $X \notin {\mathbf S}$ and $\rho^{AX}_{i+1}$ is ${\mathbf S}$-active in $G_{i+1}$ because, otherwise, $\rho^{AB}_{i+1}$ would not be ${\mathbf S}$-active in $G_{i+1}$. Then, there is an ${\mathbf S}$-active route $\rho^{AX}_{\alpha}$ between $A$ and $X$ in $G_{\alpha}$ by the induction hypothesis. Note that $Y \leftarrow Z$ is in $G_{\alpha}$ because (i) $Y$ and $Z$ are adjacent in $G_{\alpha}$ since $I(G_{\alpha}) \subseteq I(G_i)$, and (ii) $Z$ is to the left of $Y$ in $\alpha$ (recall line 4). Note also that $X$ and $Z$ are adjacent in $G_{\alpha}$ since $I(G_{\alpha}) \subseteq I(G_i)$. This and the fact that $Y \leftarrow Z$ is in $G_{\alpha}$ imply that $X$ and $Y$ are adjacent in $G_{\alpha}$ because, otherwise, $X \nci_{G_i} Y | {\mathbf U}$ but $X \ci_{G_{\alpha}} Y | {\mathbf U}$ for some ${\mathbf U} \subseteq {\mathbf V}$ such that $Z \in {\mathbf U}$, which would contradict that $I(G_{\alpha}) \subseteq I(G_i)$. In fact, $X \rightarrow Y$ is in $G_{\alpha}$. To see it, recall that Method B2 percolates left-to-right in $\beta$ one by one all the nodes in the reverse order in which they appear in $\alpha$. Method B2 is currently percolating $Y$ and, thus, the nodes to the right of $Y$ in $\alpha$ are to right of $Y$ in $\beta$ too. If $X \leftarrow Y$ were in $G_{\alpha}$ then $X$ would be to the right of $Y$ in $\alpha$ and, thus, $X$ would be to the right of $Y$ in $\beta$ too. However, this would contradict the fact that $X$ is to the left of $Y$ in $\beta$, which follows from the fact that $X$ is to the left of $Z$ in $\beta$ because $G_i$ is consistent with $\beta$, and the fact that $Y$ is the left neighbor of $Z$ in $\beta$ (recall line 3). Thus, $X \rightarrow Y$ is in $G_{\alpha}$. We now consider three cases.

\begin{description}

\item[Case 2.1] Assume that $Y \in {\mathbf S}$ and $\rho^{YB}_{i+1} = Y \leftarrow X \cup \rho^{XB}_{i+1}$. Note that $\rho^{XB}_{i+1}$ is ${\mathbf S}$-active in $G_{i+1}$ because, otherwise, $\rho^{AB}_{i+1}$ would not be ${\mathbf S}$-active in $G_{i+1}$. Then, there is an ${\mathbf S}$-active route $\rho^{XB}_{\alpha}$ between $X$ and $B$ in $G_{\alpha}$ by the induction hypothesis. Then, $\rho^{AX}_{\alpha} \cup X \rightarrow Y \leftarrow X \cup \rho^{XB}_{\alpha}$ is an ${\mathbf S}$-active route between $A$ and $B$ in $G_{\alpha}$.

\item[Case 2.2] Assume that $Y \in {\mathbf S}$ and $\rho^{YB}_{i+1} = Y \leftarrow W \cup \rho^{WB}_{i+1}$.\footnote{Note that maybe $W=B$. Note also that $W \neq X$, because the case where $W=X$ is covered by Case 2.1.} Note that $W \notin {\mathbf S}$ and $\rho^{WB}_{i+1}$ is ${\mathbf S}$-active in $G_{i+1}$ because, otherwise, $\rho^{AB}_{i+1}$ would not be ${\mathbf S}$-active in $G_{i+1}$. Then, there is an ${\mathbf S}$-active route $\rho^{WB}_{\alpha}$ between $W$ and $B$ in $G_{\alpha}$ by the induction hypothesis. Note also that $Y \leftarrow W$ is in $G_{\alpha}$. To see it, note that $Y$ and $W$ are adjacent in $G_{\alpha}$ since $I(G_{\alpha}) \subseteq I(G_i)$. Recall that the nodes to the right of $Y$ in $\alpha$ are to right of $Y$ in $\beta$ too. If $Y \rightarrow W$ were in $G_{\alpha}$ then $W$ would be to the right of $Y$ in $\alpha$ and, thus, $W$ would be to the right of $Y$ in $\beta$ too. However, this would contradict the fact that $W$ is to the left of $Y$ in $\beta$, which follows from the fact that $G_i$ is consistent with $\beta$. Thus, $Y \leftarrow W$ is in $G_{\alpha}$. Then, $\rho^{AX}_{\alpha} \cup X \rightarrow Y \leftarrow W \cup \rho^{WB}_{\alpha}$ is an ${\mathbf S}$-active route between $A$ and $B$ in $G_{\alpha}$.

\item[Case 2.3] Assume that $Y \notin {\mathbf S}$. The proof of this case is based on that of step 8 in \citep[Lemma 30]{Chickering2002}. Let $D$ denote the node that is maximal in $G_{\alpha}$ from the set of descendants of $Y$ in $G_i$. Note that $D$ is guaranteed to be unique by \citep[Lemma 29]{Chickering2002}, because $I(G_{\alpha}) \subseteq I(G_i)$. Note also that $D \neq Y$, because $Z$ is a descendant of $Y$ in $G_i$ and, as shown above, $Y \leftarrow Z$ is in $G_{\alpha}$. We now show that $D$ is a descendant of $Z$ in $G_i$. We consider three cases.

\begin{description}

\item[Case 2.3.1] Assume that $D=Z$. Then, $D$ is a descendant of $Z$ in $G_i$.

\item[Case 2.3.2] Assume that $D \neq Z$ and $D$ was a descendant of $Z$ in $G_0$. Recall that Method B2 percolates left-to-right in $\beta$ one by one all the nodes in the reverse order in which they appear in $\alpha$. Method B2 is currently percolating $Y$ and, thus, it has not yet percolated $Z$ because $Z$ is to the left of $Y$ in $\alpha$ (recall line 4). Therefore, none of the descendants of $Z$ in $G_0$ (among which is $D$) is to the left of $Z$ in $\beta$. This and the fact that $\beta$ is consistent with $G_i$ imply that $Z$ is a node that is maximal in $G_i$ from the set of descendants of $Z$ in $G_0$. Actually, $Z$ is the only such node by \citep[Lemma 29]{Chickering2002}, because $I(G_i) \subseteq I(G_0)$. Then, the descendants of $Z$ in $G_0$ are descendant of $Z$ in $G_i$ too. Thus, $D$ is a descendant of $Z$ in $G_i$.

\item[Case 2.3.3] Assume that $D \neq Z$ and $D$ was not a descendant of $Z$ in $G_0$. As shown in Case 2.3.2, the descendants of $Z$ in $G_0$ are descendant of $Z$ in $G_i$ too. Therefore, none of the descendants of $Z$ in $G_0$ was to the left of $D$ in $\alpha$ because, otherwise, some descendant of $Z$ and thus of $Y$ in $G_i$ would be to the left of $D$ in $\alpha$, which would contradict the definition of $D$. This and the fact that $D$ was not a descendant of $Z$ in $G_0$ imply that $D$ was still in $G'$ when $Z$ became a sink node of $G'$ in Construct $\beta$ (recall Figure \ref{fig:methodsab}). Therefore, Construct $\beta$ added $D$ to $\beta$ after having added $Z$ (recall lines 3-4), because $D$ is to the left of $Z$ in $\alpha$ by definition of $D$.\footnote{Note that this statement is true thanks to our correction of Construct $\beta$.} For the same reason, Construct $\beta$ did not interchange $D$ and $Z$ in $\beta$ afterwards (recall line 6). For the same reason, Method B2 has not interchanged $D$ and $Z$ in $\beta$ (recall line 4). Thus, $D$ is currently still to the left of $Z$ in $\beta$, which implies that $D$ is to the left of $Y$ in $\beta$, because $Y$ is the left neighbor of $Z$ in $\beta$ (recall line 3). However, this contradicts the fact that $G_i$ is consistent with $\beta$, because $D$ is a descendant of $Y$ in $G_i$. Thus, this case never occurs.

\end{description}

We continue with the proof of Case 2.3. Note that $Y \notin {\mathbf S}$ implies that $\rho^{YB}_{i+1}$ is ${\mathbf S}$-active in $G_{i+1}$ because, otherwise, $\rho^{AB}_{i+1}$ would not be ${\mathbf S}$-active in $G_{i+1}$. Note also that no descendant of $Z$ in $G_i$ is in ${\mathbf S}$ because, otherwise, there would be an ${\mathbf S}$-active route $\rho^{XY}_i$ between $X$ and $Y$ in $G_i$ and, thus, $\rho^{AX}_{i+1} \cup \rho^{XY}_i \cup \rho^{YB}_{i+1}$ would be an ${\mathbf S}$-active route between $A$ and $B$ in $G_{i+1}$ that would violate the minimality property of $\rho^{AB}_{i+1}$. This implies that $D \notin {\mathbf S}$ because, as shown above, $D$ is a descendant of $Z$ in $G_i$. It also implies that there is an ${\mathbf S}$-active descending route $\rho^{ZD}_i$ from $Z$ to $D$ in $G_i$. Then, $\rho^{AX}_{i+1} \cup X \rightarrow Z \cup \rho^{ZD}_i$ is an ${\mathbf S}$-active route between $A$ and $D$ in $G_{i+1}$. Likewise, $\rho^{BY}_{i+1} \cup Y \rightarrow Z \cup \rho^{ZD}_i$ is an ${\mathbf S}$-active route between $B$ and $D$ in $G_{i+1}$, where $\rho^{BY}_{i+1}$ denotes the route resulting from reversing $\rho^{YB}_{i+1}$. Therefore, there are ${\mathbf S}$-active routes $\rho^{AD}_{\alpha}$ and $\rho^{BD}_{\alpha}$ between $A$ and $D$ and between $B$ and $D$ in $G_{\alpha}$ by the induction hypothesis. 

Consider the subroute of $\rho^{AB}_{i+1}$ that starts with the arc $X \rightarrow Y$ and continues in the direction of this arc until it reaches a node $E$ such that $E=B$ or $E \in {\mathbf S}$. Note that $E$ is a descendant of $Y$ in $G_i$ and, thus, $E$ is a descendant of $D$ in $G_{\alpha}$ by definition of $D$. Let $\rho^{DE}_{\alpha}$ denote the descending route from $D$ to $E$ in $G_{\alpha}$. Assume without loss of generality that $G_{\alpha}$ has no descending route from $D$ to $B$ or to a node in ${\mathbf S}$ that is shorter than $\rho^{DE}_{\alpha}$. This implies that if $E=B$ then $\rho^{DE}_{\alpha}$ is ${\mathbf S}$-active in $G_{\alpha}$ because, as shown above, $D \notin {\mathbf S}$. Thus, $\rho^{AD}_{\alpha} \cup \rho^{DE}_{\alpha}$ is an ${\mathbf S}$-active route between $A$ and $B$ in $G_{\alpha}$. On the other hand, if $E \in {\mathbf S}$ then $E \neq D$ because $D \notin {\mathbf S}$. Thus, $\rho^{AD}_{\alpha} \cup \rho^{DE}_{\alpha} \cup \rho^{ED}_{\alpha} \cup \rho^{DB}_{\alpha}$ is an ${\mathbf S}$-active route between $A$ and $B$ in $G_{\alpha}$, where $\rho^{ED}_{\alpha}$ and $\rho^{DB}_{\alpha}$ denote the routes resulting from reversing $\rho^{DE}_{\alpha}$ and $\rho^{BD}_{\alpha}$.

\end{description}

\end{description}

\end{proof}

\begin{figure}[t]
\centering
\small
\begin{tabular}{rl}
\hline
\\
& \underline{Method G2H($G$, $H$)}\\
\\
& /* Given two DAGs $G$ and $H$ such that $I(H) \subseteq I(G)$, the algorithm transforms\\ 
& $G$ into $H$ by a sequence of arc additions and covered arc reversals such that\\
& after each operation in the sequence $G$ is a DAG and $I(H) \subseteq I(G)$ */\\
\\
1 & Let $\alpha$ denote a node ordering that is consistent with $H$\\
2 & $G$=Method B2($G$, $\alpha$)\\
3 & Add to $G$ the arcs that are in $H$ but not in $G$\\
\\
\hline
\\
\end{tabular}
\caption{Method G2H.}\label{fig:methodg2h}
\end{figure}

Finally, we show how the correctness of Method B2 leads to an alternative proof of the so-called Meek's conjecture \citep{Meek1997}. Given two DAGs $G$ and $H$ such that $I(H) \subseteq I(G)$, Meek's conjecture states that we can transform $G$ into $H$ by a sequence of arc additions and covered arc reversals such that after each operation in the sequence $G$ is a DAG and $I(H) \subseteq I(G)$. The importance of Meek's conjecture lies in that it allows to develop efficient and asymptotically correct algorithms for learning BNs from data under mild assumptions \citep{Chickering2002,ChickeringandMeek2002,Meek1997,Nielsenetal.2003}. Meek's conjecture was proven to be true in \citep[Theorem 4]{Chickering2002} by developing an algorithm that constructs a valid sequence of arc additions and covered arc reversals. We propose an alternative algorithm to construct such a sequence. The pseudocode of our algorithm, called Method G2H, can be seen in Figure \ref{fig:methodg2h}. The following corollary proves that Method G2H is correct.

\begin{cor}
Given two DAGs $G$ and $H$ such that $I(H) \subseteq I(G)$, Method G2H($G$, $H$) transforms $G$ into $H$ by a sequence of arc additions and covered arc reversals such that after each operation in the sequence $G$ is a DAG and $I(H) \subseteq I(G)$.
\end{cor}

\begin{proof}

Note from Method G2H's line 1 that $\alpha$ denotes a node ordering that is consistent with $H$. Let $G_{\alpha}$ denote the MDI map of $G$ relative to $\alpha$. Recall that $G_{\alpha}$ is guaranteed to be unique because $I(G)$ is a graphoid. Note that $I(H) \subseteq I(G)$ implies that $G_{\alpha}$ is a subgraph of $H$. To see it, note that $I(H) \subseteq I(G)$ implies that we can obtain a MDI map of $G$ relative to $\alpha$ by just removing arcs from $H$. However, $G_{\alpha}$ is the only MDI map of $G$ relative to $\alpha$.

Then, it follows from the proof of Theorem \ref{the:correctness} that Method G2H's line 2 transforms $G$ into $G_{\alpha}$ by a sequence of arc additions and covered arc reversals, and that after each operation in the sequence $G$ is a DAG and $I(G_{\alpha}) \subseteq I(G)$. Thus, after each operation in the sequence $I(H) \subseteq I(G)$ because $I(H) \subseteq I(G_{\alpha})$ since, as shown above, $G_{\alpha}$ is a subgraph of $H$. Moreover, Method G2H's line 3 transforms $G$ from $G_{\alpha}$ to $H$ by a sequence of arc additions. Of course, after each arc addition $G$ is a DAG and $I(H) \subseteq I(G)$ because $G_{\alpha}$ is a subgraph of $H$.

\end{proof}

\section{The Corrected Methods A and B are Efficient}\label{sec:efficiency}

In this section, we show that Methods A and B are more efficient than any other solution to the same problem we can think of. Let $n$ and $a$ denote, respectively, the number of nodes and arcs in $G$. Moreover, let us assume hereinafter that a DAG is implemented as an adjacency matrix, whereas a node ordering is implemented as an array with an entry per node indicating the position of the node in the ordering. Since $I(G)$ is a graphoid, the first solution we can think of consists in applying the following characterization of $G_{\alpha}$: For each node $A$, $Pa_{G_{\alpha}}(A)$ is the smallest subset ${\mathbf X} \subseteq Pre_{\alpha}(A)$ such that $A \ci_G Pre_{\alpha}(A) \setminus {\mathbf X} | {\mathbf X}$. This solution implies evaluating for each node $A$ all the $O(2^n)$ subsets of $Pre_{\alpha}(A)$. Evaluating a subset implies checking a separation statement in $G$, which takes $O(a)$ time \citep[p. 530]{Geigeretal.1990}. Therefore, the overall runtime of this solution is $O(a n 2^n)$.

Since $I(G)$ satisfies the composition property in addition to the graphoid properties, a more efficient solution consists in running the incremental association Markov boundary (IAMB) algorithm \citep[Theorem 8]{Pennaetal.2007} for each node $A$ to find $Pa_{G_{\alpha}}(A)$. The IAMB algorithm first sets $Pa_{G_{\alpha}}(A)=\emptyset$ and, then, proceeds with the following two steps. The first step consists in iterating through the following line until $Pa_{G_{\alpha}}(A)$ does not change: Take any node $B \in Pre_{\alpha}(A) \setminus Pa_{G_{\alpha}}(A)$ such that $A \nci_G B | Pa_{G_{\alpha}}(A)$ and add it to $Pa_{G_{\alpha}}(A)$. The second step consists in iterating through the following line until $Pa_{G_{\alpha}}(A)$ does not change: Take any node $B \in Pa_{G_{\alpha}}(A)$ that has not been considered before and such that $A \ci_G B | Pa_{G_{\alpha}}(A) \setminus \{B\}$, and remove it from $Pa_{G_{\alpha}}(A)$. The first step of the IAMB algorithm can add $O(n)$ nodes to $Pa_{G_{\alpha}}(A)$. Each addition implies evaluating $O(n)$ candidates for the addition, since $Pre_{\alpha}(A)$ has $O(n)$ nodes. Evaluating a candidate implies checking a separation statement in $G$, which takes $O(a)$ time \citep[p. 530]{Geigeretal.1990}. Then, the first step of the IAMB algorithm runs in $O(a n^2)$ time. Similarly, the second step of the IAMB algorithm runs in $O(a n)$ time. Therefore, the IAMB algorithm runs in $O(a n^2)$ time. Since the IAMB algorithm has to be run once for each of the $n$ nodes, the overall runtime of this solution is $O(a n^3)$.

We now analyze the efficiency of Methods A and B. To be more exact, we analyze Methods A2 and B2 (recall Figure \ref{fig:methodsa2b2b3}) rather than the original Methods A and B (recall Figure \ref{fig:methodsab}), because the former are more efficient than the latter. Methods A2 and B2 run in $O(n^3)$ time. First, note that Construct $\beta$ runs in $O(n^3)$ time. The algorithm iterates $n$ times through lines 3-10 and, in each of these iterations, it iterates $O(n)$ times through lines 5-8. Moreover, line 3 takes $O(n^2)$ time, line 6 takes $O(1)$ time, and line 9 takes $O(n)$ time. Now, note that Methods A2 and B2 iterate $n$ times through lines 2-8 and, in each of these iterations, they iterate $O(n)$ times through lines 3-7. Moreover, line 4 takes $O(1)$ time, and line 5 takes $O(n)$ time because covering an arc implies updating the adjacency matrix accordingly. Consequently, Methods A and B are more efficient than any other solution to the same problem we can think of. 

Finally, we analyze the complexity of Method G2H. Method G2H runs in $O(n^3)$ time: $\alpha$ can be constructed in $O(n^3)$ time by calling Construct $\beta$($H$, $\gamma$) where $\gamma$ is any node ordering, running Method B2 takes $O(n^3)$ time, and adding to $G$ the arcs that are in $H$ but not in $G$ can be done in $O(n^2)$ time. Recall that Method G2H is an alternative to the algorithm in \citep{Chickering2002}. Unfortunately, no implementation details are provided in \citep{Chickering2002} and, thus, a comparison with the runtime of the algorithm there is not possible. However, we believe that our algorithm is more efficient.

\section{Discussion}\label{sec:discussion}

In this paper, we have studied the problem of combining several given DAGs into a consensus DAG that only represents independences all the given DAGs agree upon and that has as few parameters associated as possible. Although our definition of consensus DAG is reasonable, we would like to leave out the number of parameters associated and focus solely on the independencies represented by the consensus DAG. In other words, we would like to define the consensus DAG as the DAG that only represents independences all the given DAGs agree upon and as many of them as possible. We are currently investigating whether both definitions are equivalent. In this paper, we have proven that there may exist several non-equivalent consensus DAGs. In principle, any of them is equally good. If we were able to conclude that one represents more independencies than the rest, then we would prefer that one. In this paper, we have proven that finding a consensus DAG is NP-hard. This made us resort to heuristics to find an approximated consensus DAG. This does not mean that we discard the existence of fast super-polynomial algorithms for the general case, or polynomial algorithms for constrained cases such as when the given DAGs have bounded in-degree. This is a question that we are currently investigating. In this paper, we have considered the heuristic originally proposed by \citet{MatzkevichandAbramson1992,MatzkevichandAbramson1993a,MatzkevichandAbramson1993b}. This heuristic takes as input a node ordering, and we have shown that finding the best node ordering for the heuristic is NP-hard. We are currently investigating the application of meta-heuristics in the space of node orderings to find a good node ordering for the heuristic. Our preliminary experiments indicate that this approach is highly beneficial, and that the best node ordering almost never coincides with any of the node orderings that are consistent with some of the given DAGs.

\section*{Acknowledgments}

We thank the anonymous referees for their thorough review of this manuscript. We thank Dr. Jens D. Nielsen for valuable comments and for pointing out a mistake in one of the proofs in an earlier version of this manuscript. We thank Dag Sonntag for proof-reading this manuscript. This work is funded by the Center for Industrial Information Technology (CENIIT) and a so-called career contract at Link\"oping University.

\section*{Appendix: Proofs of Lemmas \ref{lem:a=a2}-\ref{lem:b=b2}}

\setcounter{lem}{0}

\begin{lem}
Method A($G$, $\alpha$) and Method A2($G$, $\alpha$) return the same DAG.
\end{lem}

\begin{proof}

It is evident that Methods A and A2 transform $\beta$ into $\alpha$ and, thus, that they halt at some point. We now prove that they return the same DAG. We prove this result by induction on the number of times that Method A executes line 6 before halting. It is evident that the result holds if the number of executions is one, because Methods A and A2 share line 1. Assume as induction hypothesis that the result holds for up to $k-1$ executions. We now prove it for $k$ executions. Let $Y$ and $Z$ denote the nodes involved in the first of the $k$ executions. Since the induction hypothesis applies for the remaining $k-1$ executions, the run of Method A can be summarized as
\begin{center}
\begin{tabular}{l}
If $Z \rightarrow Y$ is in $G$ then cover and reverse $Z \rightarrow Y$ in $G$\\
Interchange $Y$ and $Z$ in $\beta$\\
For $i=1$ to $n$ do\\
\hspace{0.2cm} Percolate right-to-left in $\beta$ the leftmost node in $\beta$ that has not been percolated before
\end{tabular}
\end{center}
where $n$ is the number of nodes in $G$. Now, assume that $Y$ is percolated when $i=j$. Note that the first $j-1$ percolations only involve nodes to the left of $Y$ in $\beta$. Thus, the run above is equivalent to
\begin{center}
\begin{tabular}{l}
For $i=1$ to $j-1$ do\\
\hspace{0.2cm} Percolate right-to-left in $\beta$ the leftmost node in $\beta$ that has not been percolated before\\
If $Z \rightarrow Y$ is in $G$ then cover and reverse $Z \rightarrow Y$ in $G$\\
Interchange $Y$ and $Z$ in $\beta$\\
Percolate $Y$ right-to-left in $\beta$\\
Percolate $Z$ right-to-left in $\beta$\\
For $i=j+2$ to $n$ do\\
\hspace{0.2cm} Percolate right-to-left in $\beta$ the leftmost node in $\beta$ that has not been percolated before.
\end{tabular}
\end{center}
Now, let ${\mathbf W}$ denote the nodes to the left of $Z$ in $\beta$ before the first of the $k$ executions of line 6. Note that the fact that $Y$ and $Z$ are the nodes involved in the first execution implies that the nodes in ${\mathbf W}$ are also to the left of $Z$ in $\alpha$. Note also that, when $Z$ is percolated in the latter run above, the nodes to the left of $Z$ in $\beta$ are exactly ${\mathbf W} \cup \{Y\}$. Since all the nodes in ${\mathbf W} \cup \{Y\}$ are also to the left of $Z$ in $\alpha$, the percolation of $Z$ in the latter run above does not perform any arc covering and reversal or node interchange. Thus, the latter run above is equivalent to
\begin{center}
\begin{tabular}{l}
For $i=1$ to $j-1$ do\\
\hspace{0.2cm} Percolate right-to-left in $\beta$ the leftmost node in $\beta$ that has not been percolated before\\
Percolate $Z$ right-to-left in $\beta$\\
Percolate $Y$ right-to-left in $\beta$\\
For $i=j+2$ to $n$ do\\
\hspace{0.2cm} Percolate right-to-left in $\beta$ the leftmost node in $\beta$ that has not been percolated before
\end{tabular}
\end{center}
which is exactly the run of Method A2. Consequently, Methods A and A2 return the same DAG.

\end{proof}

\begin{lem}
Method A2($G$, $\alpha$) and Method B($G$, $\alpha$) return the same DAG.
\end{lem}

\begin{proof}

We can prove the lemma in much the same way as Lemma \ref{lem:a=a2}. We simply need to replace $Y$ by $Z$ and vice versa in the proof of Lemma \ref{lem:a=a2}.

\end{proof}

\begin{lem}
Method B($G$, $\alpha$) and Method B2($G$, $\alpha$) return the same DAG.
\end{lem}

\begin{proof}

It is evident that Methods B and B2 transform $\beta$ into $\alpha$ and, thus, that they halt at some point. We now prove that they return the same DAG. We prove this result by induction on the number of times that Method B executes line 6 before halting. It is evident that the result holds if the number of executions is one, because Methods B and B2 share line 1. Assume as induction hypothesis that the result holds for up to $k-1$ executions. We now prove it for $k$ executions. Let $Y$ and $Z$ denote the nodes involved in the first of the $k$ executions. Since the induction hypothesis applies for the remaining $k-1$ executions, the run of Method B can be summarized as
\begin{center}
\begin{tabular}{l}
If $Y \rightarrow Z$ is in $G$ then cover and reverse $Y \rightarrow Z$ in $G$\\
Interchange $Y$ and $Z$ in $\beta$\\
For $i=1$ to $n$ do\\
\hspace{0.2cm} Percolate left-to-right in $\beta$ the rightmost node in $\alpha$ that has not been percolated before
\end{tabular}
\end{center}
where $n$ is the number of nodes in $G$. Now, assume that $Y$ is the $j$-th rightmost node in $\alpha$. Note that, for all $1 \leq i < j$, the $i$-th rightmost node $W_i$ in $\alpha$ is to the right of $Y$ in $\beta$ when $W_i$ is percolated in the run above. To see it, assume to the contrary that $W_i$ is to the left of $Y$ in $\beta$. This implies that $W_i$ is also to the left of $Z$ in $\beta$, because $Y$ and $Z$ are neighbors in $\beta$. However, this is a contradiction because $W_i$ would have been selected in line 2 instead of $Y$ for the first execution of line 6. Thus, the first $j-1$ percolations in the run above only involve nodes to the right of $Z$ in $\beta$. Then, the run above is equivalent to
\begin{center}
\begin{tabular}{l}
For $i=1$ to $j-1$ do\\
\hspace{0.2cm} Percolate left-to-right in $\beta$ the rightmost node in $\alpha$ that has not been percolated before\\
If $Y \rightarrow Z$ is in $G$ then cover and reverse $Y \rightarrow Z$ in $G$\\
Interchange $Y$ and $Z$ in $\beta$\\
For $i=j$ to $n$ do\\
\hspace{0.2cm} Percolate left-to-right in $\beta$ the rightmost node in $\alpha$ that has not been percolated before
\end{tabular}
\end{center}
which is exactly the run of Method B2.

\end{proof}

\end{document}